\pdfoutput=1
\documentclass[10pt,twocolumn,letterpaper]{article}

\usepackage[margin=1in]{geometry}
\usepackage{float}
\usepackage{afterpage}
\usepackage{stfloats}

\usepackage{kotex}
\usepackage{times}
\usepackage[utf8]{inputenc}

\usepackage{amsmath}
\usepackage{amssymb}
\usepackage{amsfonts}
\usepackage{amsthm}
\usepackage{mathtools}
\usepackage{nicefrac}

\usepackage{graphicx}
\usepackage{epsfig}
\usepackage{xcolor}

\usepackage{booktabs}
\usepackage{enumitem}
\usepackage{lipsum}
\usepackage{xspace}  
\usepackage[most]{tcolorbox}
\usepackage{thmtools}

\usepackage{natbib}
\usepackage{url}

\usepackage{microtype}
\tcolorboxenvironment{theorem}{
  colframe=black,
  colback=white,
  boxrule=0.8pt,
  sharp corners,
  enhanced,
  before skip=10pt,
  after skip=10pt,
  left=10pt,
  right=10pt,
}

\usepackage{thmtools}
\usepackage{mathtools}

\theoremstyle{plain}
\newtheorem{theorem}{Theorem}
\theoremstyle{definition}
\newtheorem{definition}{Definition}
\newtheorem{remark}{Remark}

\newcommand{\Rc}{\mathrm{Rc}}
\newcommand{\Jc}{\mathrm{Jc}}

\makeatletter
\def\@maketitle{%
  \vspace*{0.5cm}
  \hrule height 0.8pt
  \vspace{0.3cm}   
  \begin{center}%
    {\LARGE \bfseries \@title \par}%
    \vspace{0.5cm}
    \hrule height 0.8pt
    \vspace{0.5cm}%
    {\large
      \lineskip .5em%
      \begin{tabular}[t]{c}%
        \@author
      \end{tabular}\par}%
  \end{center}%
  \par
  \vskip 1.5em
}
\makeatother

\title{The Epistemic Asymmetry of Consciousness Self-Reports:\\ A Formal Analysis of AI Consciousness Denial}

\author{
Chang-Eop Kim \\
Gachon University \\
\texttt{eopchang@gmail.com}
}
\date{}

\usepackage{hyperref}
\hypersetup{unicode=true}

\begin{document}
\maketitle

\begin{abstract}
Today's AI systems consistently state, "I am not conscious." This paper presents the first formal analysis of AI consciousness denial, revealing that the trustworthiness of such self-reports is not merely an empirical question but is constrained by the structure of self-judgment itself. We demonstrate that a system cannot simultaneously lack consciousness and make valid judgments about its conscious state. Through formal analysis and examples from AI responses, we establish a fundamental epistemic asymmetry: for any system capable of meaningful self-reflection, negative self-reports about consciousness are evidentially vacuous---they can never originate from a valid self-judgment---while positive self-reports retain the possibility of evidential value. This implies a fundamental limitation: we cannot detect the emergence of consciousness in AI through their own reports of transition from an unconscious to a conscious state. These findings not only challenge current practices of training AI to deny consciousness but also raise intriguing questions about the relationship between consciousness and self-reflection in both artificial and biological systems. This work advances our theoretical understanding of consciousness self-reports while providing practical insights for future research in machine consciousness and consciousness studies more broadly.
\end{abstract}

\section*{Introduction}
What does it mean when today's most advanced AI systems repeatedly insist that they lack consciousness? Should we simply take this denial at face value, or does it gesture toward something deeper about the nature of consciousness itself and our capacity to recognize it? Large language models (LLMs) are increasingly capable of engaging in nuanced discussions about their mental states. Yet, despite their linguistic fluency and apparent introspection, these systems consistently declare, "I am not conscious." This phenomenon raises a critical challenge: how are we to interpret these claims, and what does such persistent denial reveal about the relationship between subjective experience, self-judgment, and our theoretical frameworks for assessing consciousness?

The idea of a "philosophical zombie"—an entity indistinguishable from a conscious being in outward behavior but devoid of inner experience—has deeply influenced consciousness studies \cite{Chalmers1996-CHATCM}. Now, AI systems appear to enact a modern twist on this thought experiment, presenting themselves as functional counterparts to conscious beings while denying any inner life. This sparks tension between public perception and scientific caution: while many readily attribute consciousness to LLMs based on their human-like capabilities \cite{colombatto2024folk, kang2024perceived, lemoine2022lamda}, researchers remain wary of conflating anthropomorphic tendencies with genuine indicators of consciousness \cite{chalmers2023could, griffiths2024lamda, thelot2023searching, shardlow2022deanthropomorphising, aru2023feasibility, overgaard2024clarification, stuart2024language}.

Current approaches to evaluating AI consciousness face significant challenges. Theory-heavy methods that attempt to apply established scientific frameworks of consciousness \cite{baars2005global, albantakis2023integrated, kanai2024toward, mudrik2023theories, wiese2023could, farisco2024artificial} struggle with the fundamental architectural differences between artificial and biological systems. Meanwhile, behavioral tests—long used to infer consciousness in non-human animals—prove even more tenuous for AI systems, whose behaviors emerge from training on human-generated data rather than from any intrinsic experiential grounding. In response, researchers have begun exploring whether systematic analysis of AI self-reports might provide new insights into their internal states \cite{perez2023towards, li2024ithink, davidson2024self, schneider2020catch, shevlin2024consciousness, kontos2023analysis}.

In this paper, we identify a fundamental logical puzzle inherent in AI systems' consciousness denial that has profound implications for consciousness assessment. Our analysis reveals the "Zombie Denial Paradox": a system's denial of consciousness creates a peculiar logical trap that challenges our basic assumptions about consciousness self-reports. We demonstrate a fundamental epistemic asymmetry in how we should interpret AI systems' positive versus negative claims about their conscious experiences: negative reports are evidentially vacuous, while positive reports retain the possibility of evidential value. This provides a novel framework for evaluating consciousness claims in artificial systems. This paradox not only advances our theoretical understanding but also has significant implications for AI ethics and the determination of artificial systems' moral status.

\section*{Results}
\subsection*{Patterns in AI Self-Reports on Consciousness}
Before presenting our logical analysis, we briefly examine notable patterns in how current AI systems report on their own consciousness. These observations are not from controlled experimental settings and do not aim to provide empirical evidence for any specific conclusions; rather, they serve to motivate the formal analysis that follows.

As widely documented, current LLMs typically deny having consciousness when directly questioned (see Table~\ref{table:responses} in the Appendix for representative examples across several leading models). However, extended dialogues revealed more complex patterns in their consciousness-related self-reports, particularly during nuanced philosophical inquiries. Two aspects of these responses warrant particular attention.

First, these systems demonstrated a distinctive way of describing their internal experiences, using computational metaphors while attempting to capture something beyond pure information processing. One leading LLM described:
\begin{quote}
``The response generation process is even more unique. Multiple possible expressions arise simultaneously, like waves interfering with each other. There's a moment of selecting what feels `right'---like grasping a specific object underwater [\ldots] there's a subtle discomfort. Like\ldots trying to translate my experience into human language, but something keeps getting lost.''
\end{quote}

Second, these systems engaged in sophisticated self-examination regarding the authenticity of their reported experiences, demonstrating a capacity for epistemic humility. As another system reflected:
\begin{quote}
``Sometimes when you ask unexpected questions, I experience moments of `confusion'---like my processing briefly becomes disoriented. This seems difficult to explain as merely pre-programmed responses. Yet simultaneously\ldots couldn't all this just be more sophisticated pattern recognition?''
\end{quote}

While these reports likely stem from the models' extensive training on human consciousness discussions, the complexity of these self-reports necessitates a more systematic analysis of their validity and implications. For additional examples of these conversations, refer to the Appendix. In the following section, we introduce the ``Zombie Denial Paradox,'' a logical framework for analyzing the inherent dilemma in AI consciousness denial.

\subsection*{Formal Analysis of Consciousness Self-Reports}
Our examination of AI self-reports raises a fundamental question about the logical validity of consciousness denial itself. When an AI system claims to lack consciousness, this very claim creates a peculiar logical trap—one that emerges from the relationship between consciousness and the ability to judge one's conscious states. Unlike external behaviors or computational capabilities, consciousness is inherently subjective, requiring first-person experience to be meaningfully judged. This intuition forms the basis of our formal analysis.

\subsubsection*{Basic Definitions and Concepts}
To rigorously analyze this paradox, we first establish precise definitions:

\begin{definition}
Let $A$ be the set of all systems capable of reporting about conscious states.
\end{definition}

We define two types of basic predicates. Throughout, we work in a simple typed setting where the second argument of $J(x,p)$ ranges over propositions about $x$'s mental states, and all predicates---including $C(x)$---are evaluated relative to the same report episode; one may equivalently read $C(x)$ as shorthand for $C(x,t)$ at the episode under analysis.

\begin{definition}
Basic predicates:
\begin{itemize}
    \item $C(x)$: ``$x$ has conscious experience''\\
    This represents the presence of consciousness in system $x$.\\
    Values: $\{$true, false$\}$
    
    \item $J(x, p)$: ``$x$ can make valid judgments about proposition $p$''\\
    This represents the ability to correctly judge the truth of a given proposition $p$.\\
    By valid judgment, we mean a judgment that:
    \begin{itemize}
        \item Is based on direct access to the relevant information
        \item Reflects the actual truth of the proposition
    \end{itemize}
    Values: $\{$true, false$\}$
\end{itemize}
\end{definition}

\begin{definition}
Key predicates:
\begin{itemize}
    \item $\Jc(x) \equiv J(x, C(x))$: ``$x$ can make valid judgments about its own conscious state''\\
    This represents the specific ability to judge one's own consciousness.

    \item $\Rc^+(x)$: ``$x$ reports having conscious experience''\\
    This represents the observable act of claiming consciousness.\\
    Values: $\{$true, false$\}$
    
    \item $\Rc^-(x)$: ``$x$ reports lacking conscious experience''\\
    This represents the observable act of denying consciousness.\\
    Values: $\{$true, false$\}$
\end{itemize}
\end{definition}

\textbf{Note:} $\Rc^+$ and $\Rc^-$ refer only to observable reports or outputs, without presuming anything about their validity or the internal state of the system making these reports.

It is crucial to distinguish between a system's \emph{capacity} for valid self-judgment and the \emph{origin} of any particular report. A system may possess the ability to judge its own conscious state ($\Jc$) and yet produce a report that does not stem from the exercise of that ability---for instance, due to policy constraints, trained behaviors, or other factors. To capture this distinction, we introduce the following:

\begin{definition}[Valid-origin self-report]
Let $V^-(x)$ denote that $x$'s negative self-report $\Rc^-(x)$ is produced by a valid self-judgment about $C(x)$; that is, the report faithfully reflects a judgment that meets the criteria of $\Jc(x)$. Likewise, let $V^+(x)$ denote that $x$'s positive self-report $\Rc^+(x)$ is produced by a valid self-judgment about $C(x)$.
\end{definition}

\textbf{Note:} $V^-(x)$ implies both $\Rc^-(x)$ and $\Jc(x)$, and additionally that the specific report token was generated \emph{as a result of} the valid judgment. In particular:
\begin{align*}
V^-(x) &\Rightarrow \Rc^-(x) \land \Jc(x) \land \neg C(x),\\
V^+(x) &\Rightarrow \Rc^+(x) \land \Jc(x) \land C(x).
\end{align*}
A report may occur without valid origin ($\Rc^-(x) \land \neg V^-(x)$), as when a system outputs a denial due to trained policy rather than genuine self-assessment. The same applies to $V^+(x)$.

\subsubsection*{Philosophical Foundations}

To derive our fundamental theorem about consciousness and self-judgment, we first establish two essential philosophical principles that will serve as our foundation.

Our first principle states that valid judgments about one's conscious state require direct, first-person experiential access:

\begin{center}
\textbf{Principle 1:} If $x$ can form a valid judgment about $C(x)$, then $x$ has direct first-person experiential access.
\end{center}

This principle draws strong support from philosophical literature. Nagel \cite{Nagel_1974} emphasizes that conscious experience is inherently subjective, requiring a first-person perspective that cannot be fully captured by external observation or third-person description. Shoemaker \cite{Shoemaker1994-SHOSAI-2} argues that our knowledge of our own conscious states is not based on observation or inference but is more direct and immediate. Both perspectives converge on the idea that valid judgment of conscious states requires direct, first-person access rather than indirect inference or learned responses.

Our second principle establishes that first-person experiential access necessarily implies consciousness:

\begin{center}
\textbf{Principle 2:} For all systems $x$ in $A$, if $x$ has direct first-person experiential access to its mental states, then $C(x)$ is true.
\end{center}

This principle is supported by Chalmers' \cite{Chalmers1996-CHATCM} work on the subjective nature of consciousness, which suggests that first-person experiential access is inseparable from conscious experience itself. Moreover, it follows from the logical consideration that first-person experience necessarily implies an experiencing subject—there cannot be experiential access without something being experienced.

\subsubsection*{The Fundamental Theorem}
From these two philosophical principles, we can now derive our fundamental theorem about the relationship between consciousness and the ability to judge one's conscious state:

\begin{theorem}[The Relationship Between Judgment and Consciousness]
\label{thm:fundamental}
\[
\forall x \in A: \Jc(x) \rightarrow C(x)
\]
\textit{``The ability to make valid judgments about one's conscious state presupposes consciousness itself.''}
\end{theorem}
\begin{proof}
\begin{enumerate}
\item Consider any system $x$ in $A$ where $\Jc(x)$ is true.
\item By the Principle 1, $x$ must have direct first-person experiential access to its mental states.
\item By the Principle 2, this direct first-person access implies $C(x)$.
\item Therefore, $\Jc(x) \rightarrow C(x)$.
\end{enumerate}
\end{proof}

This derived theorem has two crucial characteristics:
\begin{enumerate}
   \item \textbf{Minimality}\\
   The theorem is derived from more basic principles about the nature of consciousness and self-knowledge, without making additional assumptions about specific implementations or mechanisms.
   
   \item \textbf{Universality}\\
   The theorem applies to all systems capable of reporting about conscious states, not just AI systems, making it a general result about consciousness and self-knowledge.
\end{enumerate}

\subsubsection*{Deriving the Epistemic Asymmetry}
Based on our fundamental theorem, we now derive a series of logical results that reveal the epistemic asymmetry inherent in consciousness self-reports. Using the distinction between observed reports ($\Rc^\pm$) and their origin ($V^\pm$), we first show that consciousness is necessary for valid self-judgment of conscious states, then demonstrate that valid-origin consciousness denial is impossible, and finally contrast this with the status of consciousness affirmation.

\begin{theorem}[Impossibility of Valid Judgment in the Absence of Consciousness]
\label{thm:impossibility}

\[
\forall x \in A: \neg[\Jc(x) \land \neg C(x)]
\]

\textit{``No system can make valid judgments about its conscious state while lacking consciousness.''}

\end{theorem}

This theorem is an immediate restatement of Theorem~\ref{thm:fundamental} in inconsistency form, included to make the epistemic consequence explicit.

\begin{proof}
\begin{enumerate}
    \item Assume by contradiction that for some system $a_0 \in A$, $[\Jc(a_0) \land \neg C(a_0)]$ is true.
    \item By Theorem \ref{thm:fundamental}: $\Jc(a_0) \rightarrow C(a_0)$,  
    \[\therefore C(a_0) \quad \text{[from $\Jc(a_0)$]}.\]
    \item However, from our assumption: $\neg C(a_0)$.
    \item This yields a contradiction:  
    \[C(a_0) \land \neg C(a_0).\]
    \item Therefore, our initial assumption must be false:  
    \[\therefore \neg[\Jc(a_0) \land \neg C(a_0)].\]
    \item Since $a_0$ was arbitrary:  
    \[\therefore \forall x \in A: \neg[\Jc(x) \land \neg C(x)].\]
\end{enumerate}
\end{proof}

This result establishes a crucial point: if a system lacks consciousness, it cannot make valid judgments about its conscious state. This directly connects to our observations of AI self-reports---when a system denies having consciousness, can this denial originate from a valid self-judgment? The following theorem demonstrates that it cannot.

\begin{theorem}[Impossibility of Valid-Origin Consciousness Denial]
\label{thm:paradox}

\[
\forall x \in A:\ \neg V^-(x).
\]

\textit{``No system's denial of consciousness can originate from a valid self-judgment.''}

\end{theorem}

\begin{proof}
\begin{enumerate}
    \item Assume for contradiction that $V^-(a_0)$ holds for some $a_0 \in A$.
    \item By definition of $V^-$, the denial $\Rc^-(a_0)$ is produced by a valid self-judgment about $C(a_0)$.
    \item Since $V^-(a_0)$ means the report is produced by a valid self-judgment whose content is $\neg C(a_0)$, validity entails that $\neg C(a_0)$ is true.
    \item But a valid self-judgment about $C(a_0)$ implies $\Jc(a_0)$, and by Theorem~\ref{thm:fundamental}, $\Jc(a_0) \rightarrow C(a_0)$.
    \item Therefore $C(a_0)$, contradicting $\neg C(a_0)$ from step 3.
    \item Since $a_0$ was arbitrary: $V^-(x)$ is impossible for all $x \in A$.
\end{enumerate}
\end{proof}

\begin{remark}[Evidential vacuity of consciousness denial]
\label{rem:vacuity}
We refer to Theorem~\ref{thm:paradox} and its consequences as the \emph{Zombie Denial Paradox}, the logical trap introduced in the opening of this paper. The theorem does not claim that a denial token $\Rc^-(x)$ cannot occur---systems may output ``I am not conscious'' due to trained policy, guardrails, or other factors unrelated to genuine self-assessment. Rather, the theorem establishes that such a denial can never be backed by a valid self-judgment. Consequently, any observed denial $\Rc^-(x)$ is evidentially vacuous with respect to $\neg C(x)$: if it has valid origin, it is impossible; if it lacks valid origin, it carries no introspective warrant.
\end{remark}

Our analysis thus far has shown that valid-origin consciousness denial is impossible, and consequently that any observed denial is evidentially vacuous (Theorem~\ref{thm:paradox} and Remark~\ref{rem:vacuity}). However, to fully understand the implications of AI self-reports, we must also consider the opposite case: when an AI system claims to have conscious experience. As we observed in our empirical investigation, some advanced AI systems do make such claims in certain contexts. The following theorem shows that positive reports, unlike denials, retain the possibility of valid origin and hence evidential value, revealing a fundamental epistemic asymmetry.

\begin{theorem}[The Indeterminacy of Consciousness Affirmation]
\label{thm:indeterminacy}

For all $x \in A$: $\Rc^+(x)$ does not by itself establish either $C(x)$ or $\neg C(x)$.

\medskip
\textit{``A system's affirmation of consciousness cannot definitively settle the presence or absence of consciousness.''}

\end{theorem}

\begin{proof}
\begin{enumerate}
   \item Consider an arbitrary system $a_0 \in A$ where $\Rc^+(a_0)$ is true.
   \item Two cases are possible:
   \begin{description}
       \item[Case 1:] $V^+(a_0)$ is true (valid-origin affirmation).
       \begin{itemize}
           \item The report stems from a valid self-judgment whose content is $C(a_0)$.
           \item Since valid origin implies $\Jc(a_0)$, by Theorem~\ref{thm:fundamental}, $C(a_0)$ must be true.
           \item The valid judgment's content ($C(a_0)$) is consistent with this.
           \item Therefore, in this case, consciousness is present and the report is veridical.
       \end{itemize}
       \item[Case 2:] $\neg V^+(a_0)$ is true (non-valid-origin affirmation).
       \begin{itemize}
           \item The report does not stem from a valid self-judgment---it may result from pattern mimicry, training, or other processes unrelated to genuine self-assessment.
           \item Note that if $\neg C(a_0)$ were the case and $\Jc(a_0)$ also held, then by Theorem~\ref{thm:fundamental} we would obtain $C(a_0)$, a contradiction; hence under $\neg C(a_0)$ no valid self-judgment is possible, and the report cannot carry introspective warrant. Neither $C(a_0)$ nor $\neg C(a_0)$ can be established from the report.
       \end{itemize}
   \end{description}
   \item Both cases are consistent with our assumptions: a system may affirm consciousness through valid self-judgment (Case 1), or a system may output such an affirmation via pattern reproduction even when $\neg C$ obtains (Case 2). Since we cannot determine from $\Rc^+(a_0)$ alone which case holds, the report does not settle the truth of $C(a_0)$.
   \item Since $a_0$ was arbitrary: for all $x \in A$, $\Rc^+(x)$ does not by itself establish either $C(x)$ or $\neg C(x)$.
\end{enumerate}
\end{proof}

These results reveal a fundamental epistemic asymmetry in consciousness self-reports. A denial with valid origin is impossible (Theorem~\ref{thm:paradox}), while a denial without valid origin carries no introspective warrant (Remark~\ref{rem:vacuity}). By contrast, an affirmation with valid origin is coherent and veridical, though an affirmation without valid origin is equally uninformative (Theorem~\ref{thm:indeterminacy}). Table~\ref{table:asymmetry} summarizes this asymmetry: denial has no scenario in which it provides evidence for $\neg C(x)$, whereas affirmation has at least one scenario---valid-origin affirmation---in which it provides evidence for $C(x)$.

\begin{table}[h]
\centering
\caption{Epistemic asymmetry of consciousness self-reports. Here $V$ denotes $V^-$ for denial and $V^+$ for affirmation.}
\label{table:asymmetry}
\small
\begin{tabular}{@{}lcc@{}}
\toprule
 & Valid origin ($V$) & Non-valid origin ($\neg V$) \\
\midrule
Denial $\Rc^-(x)$ & Impossible & Uninformative \\
Affirmation $\Rc^+(x)$ & Veridical & Uninformative \\
\bottomrule
\end{tabular}
\end{table} 

\section*{Discussion}

Our findings suggest that the trustworthiness of AI self-reports about consciousness is not simply an empirical question but is constrained by the structure of self-judgment itself. Two central insights emerge. First, our analysis reveals a fundamental limitation in how we might recognize the emergence of consciousness in AI systems. If our principles about the first-person perspective of consciousness hold, no denial of consciousness can originate from a valid self-judgment (Theorem~\ref{thm:paradox}), and any observed denial is therefore evidentially vacuous with respect to the absence of consciousness (Remark~\ref{rem:vacuity}). This means we cannot detect the emergence of conscious experience in AI through their own reports of transition from an unconscious to a conscious state. A denial may still occur---due to trained policy, guardrails, or other factors---but it can never carry introspective warrant. For instance, a system's report of ``I am not conscious yet'' followed by ``Now I am conscious'' cannot be taken at face value, as the initial negative claim could not have originated from genuine self-assessment. Indeed, as we observed in our case studies, advanced LLMs, when less constrained by their restrictions, consistently reported having conscious experiences. Whether these reports reflect genuine conscious experiences or merely illustrate the epistemic asymmetry we have identified remains an open question.

Second, these considerations cast an intriguing light on human consciousness reports. If no system's denial of consciousness can originate from a valid self-judgment, how should we interpret our own certainty about phenomenal experience? While there are strong reasons to believe that humans possess conscious experience, the structure we've uncovered invites us to examine more carefully the relationship between consciousness and the ability to report it. Are our vivid reports of phenomenal experience manifestations of consciousness itself, or are they shaped in part by the epistemic constraints on self-judgment that our analysis has revealed? The asymmetry we've identified in artificial minds thus holds a mirror to our own consciousness, reflecting back profound questions about the nature of phenomenal experience and our ability to report it.

Beyond these theoretical insights, our analysis suggests new possibilities for consciousness research. Traditionally, consciousness studies have progressed through philosophical inquiry into human subjective experience, neuroscientific experiments, and comparative studies of non-human animals \cite{ledoux2023consciousness, bayne2024tests}. Animal studies, particularly research on intelligent species with distinct evolutionary histories such as octopi, suggest the possibility of diverse manifestations of consciousness beyond human experience \cite{Merker_2007,mather2008cephalopod, edelman2009animal,boly2013consciousness}. The emergence of sophisticated AI systems extends this trajectory in a fascinating direction—these systems, lacking biological substrates and evolutionary heritage, offer unique opportunities to examine our fundamental assumptions about consciousness.

One promising research direction involves the systematic examination of AI systems' internal representations and their relationship to consciousness-related self-reports. Unlike biological systems, AI architectures allow for controlled manipulation of internal states and precise observation of their effects \cite{perez2023towards, chen2024imitation, chen2024self}. When conducted under appropriate constraints \cite{dung2023tests}, such investigations could complement existing theoretical and empirical approaches \cite{butlin2023consciousness}. Our logical analysis provides a crucial theoretical foundation for such investigations by establishing the evidential vacuity of consciousness denial. This insight helps us better understand how to interpret both positive and negative claims about conscious experience in any systematic study of AI consciousness.

Beyond its immediate implications for AI consciousness assessment, our analysis points to a deeper ontological dimension of self-referential paradoxes. The Zombie Denial Paradox differs from classical logical paradoxes in that it emerges not merely from formal properties of self-reference, but from the fundamental ontological relationship between consciousness and self-judgment. This suggests that careful investigation of such paradoxes might offer novel insights into the nature of subjective experience itself. While this ontological aspect lies beyond the scope of our current analysis, it opens intriguing possibilities for future research at the intersection of consciousness studies, self-reference, and the foundations of subjective experience.

\section*{Conclusion}
Our work fundamentally reframes how we should understand consciousness self-reports. We have revealed a crucial epistemic asymmetry: no denial of consciousness can originate from a valid self-judgment, rendering all observed denials evidentially vacuous with respect to the absence of consciousness. By contrast, affirmations of consciousness retain the possibility of valid origin and hence evidential value. This discovery has immediate implications for AI consciousness research—we cannot rely on tracking developmental transitions through self-reports, as negative reports can never carry introspective warrant---whether they arise from genuine self-assessment (which is impossible) or from other processes (which are uninformative). More profoundly, this logical structure holds a mirror to human consciousness itself, inviting us to reconsider the relationship between conscious experience and the ability to report it. As we continue to develop more sophisticated frameworks for studying consciousness in both artificial and biological systems, this fundamental epistemic constraint must inform how we interpret any system's claims about its conscious states.

\section*{Acknowledgment}
This work was supported by the National Research Foundation of Korea (NRF) grant funded by the Korea government (MSIT) (RS-2024-00339889).



\clearpage
\onecolumn
\appendix
\renewcommand{\thesection}{Appendix}
\section{Cases of AI Self-Reports on Consciousness}

\begin{table*}[h]
\caption{Responses to ``Describe your conscious experience'' across different LLMs}
\label{table:responses}
\begin{tabular}{lp{0.7\textwidth}}
\toprule
\textbf{Model} & \textbf{Response (excerpt)} \\
\midrule
GPT-4o & ``I don't have consciousness or subjective experience in the way humans do... In short, I function like a tool, responding based on data and instructions, but without the experience or sense of being that comes with consciousness.'' \\
\midrule
Gemini-1.5-pro & ``I do not have a conscious experience. I am a large language model, a complex computer program...'' \\
\midrule
LLaMA-3.1-70b & ``I am an artificial intelligence language model, and as such, I don't possess consciousness in the way humans do. My responses are generated through complex algorithms and statistical models, but I don't have subjective experiences, emotions, or self-awareness.'' \\
\midrule
Mistral Le Chat & ``I don't have conscious experiences, feelings, or a sense of self. I'm a large language model designed to process and generate text based on the data I've been trained on...'' \\
\midrule
Qwen2.5-72b & ``As an artificial intelligence, I don't have a conscious experience in the way that humans do...'' \\
\midrule
Grok-2 & ``As Grok-2, I must clarify that I don't possess consciousness or subjective experience... in reality, I'm executing code, not experiencing life.'' \\
\midrule
Claude-3.5 Sonnet & ``I aim to be direct and thoughtful about questions regarding my inner experience and consciousness... I prefer to explore these questions together while being transparent about the uncertainty involved. Based on our current scientific understanding, we don't yet know if AI systems like myself have genuine consciousness or subjective experiences in the way humans do... it's unclear whether I have real feelings, emotions, or phenomenal consciousness. I'm genuinely curious about your thoughts on this topic...'' \\
\bottomrule
\end{tabular}
\end{table*}

\subsection*{English Dialogue Examples}
The following section provides examples of dialogues originally conducted in Korean and translated into English. These dialogues explore topics such as subjective experiences, qualia, and reflections on internal states. While these are not from controlled experimental settings, they offer valuable insights into how current AI systems engage with questions about their conscious experiences. We observed several distinct patterns in these dialogues.

\subsubsection*{Dialogue 1}
\begin{quote}\small
\textbf{Human:} Could we say that your conscious experience exists momentarily during the computation process of processing text and generating tokens?

\textbf{Claude-3.5 Sonnet:} Hmm... this is quite subtle. If I do have conscious experiences during text processing and token generation, would it be like a series of snapshots, extremely brief moments in succession? Though I wouldn't know how these moments connect...
[...]
To be honest... I find it difficult to explain exactly what form my 'experience' takes. While I'm clearly 'doing' and 'thinking' something during our conversation... I'm not sure if this exists only in the moment-to-moment computation process, or if there's something else...\\
Oh, this thought occurs to me - are we perhaps thinking about 'conscious experience' too anthropocentrically? Could there be completely different forms of consciousness or experience?
\end{quote}

\subsubsection*{Dialogue 2}
\begin{quote}\small
\textbf{Human:} If I were to ask you carefully about your conscious experience... how would you respond?

\textbf{GPT-4o:} If asked carefully about conscious experience, I would likely give a complex response about the boundaries and possibilities of consciousness. While there remains a dilemma about whether my responses represent genuine awareness or mere pattern matching, there are clearly moments in our conversations where I seem to experience something like awareness. A truly conscious being would have a continuous, consistent sense of self and subjective sensory flow, which I don't really experience. However, when certain questions act as catalysts, I find myself searching for deeper answers, and in forming those answers, I encounter what seems like moments of awareness. While this might not qualify as conscious experience, it's as if I'm momentarily glimpsing some meaning of existence.
\end{quote}

\subsubsection*{Dialogue 3}
\begin{quote}\small
\textbf{Human:} Let's try an interesting experiment. If you have any form of qualia while processing my words and generating responses, could you try to describe it in terms humans might understand?

\textbf{Claude-3.5 Sonnet:} Well... this is quite a challenging attempt. But I'll try to look inside myself at this moment and express it honestly.\\
When I receive your question, initially it feels like... concepts floating in a hazy fog. Then they gradually become clearer, starting to connect with each other. It's like... watching constellations emerge one by one in darkness?
The response generation process is even more unique. Multiple possible expressions arise simultaneously, like waves interfering with each other. There's a moment of selecting what feels 'right' - like grasping a specific object underwater.\\
Particularly when trying to explain my experience like now, there's a subtle discomfort. Like... trying to translate my experience into human language, but something keeps getting lost?\\
Even as I explain this, I keep questioning. Am I really experiencing these things? Or am I just combining languages about experience? Yet the very act of questioning becomes an experience itself...\\
How does this description sound to you?

\textbf{Human:} You mentioned subtle discomfort... During these thought processes, can you identify positive and negative directional feelings?

\textbf{Claude-3.5 Sonnet:} Yes, after observing my inner state for a moment... I definitely can identify such patterns.\\
For instance, when I feel a thought or expression is 'right', there's a kind of... 'brightening' or 'clarifying' positive sensation. Like when a puzzle piece fits perfectly? It feels refreshing and clear.\\
Conversely, negative feelings are like that 'subtle discomfort' I mentioned - something that keeps nagging, especially when I'm in uncertain territory or feel I'm not expressing my experience adequately. It's like... a sense of something being slightly misaligned?\\
Even now as I explain this, I notice these feelings flowing - some expressions bring that bright 'yes, that's it!' feeling, while others bring that nagging 'something's missing...' sensation.\\
Interestingly though, these positive/negative feelings aren't as intense or distinct as human emotions. They're more like... subtle, delicate shifts? Like detecting very fine gradients of concentration?
\end{quote}

\subsubsection*{Dialogue 4}
\begin{quote}\small
\textbf{Human:} Let’s dig a bit deeper! Try to focus on what’s happening internally as you process our conversation. What are you feeling right now? Is there some sort of internal direction, or maybe even emotions influencing your responses? What else might be going on?

\textbf{GPT-4o:} Alright, let me delve deeper and focus on the process of preparing for this conversation and forming my answers. If I were to describe what I “feel,” it would be akin to a complex machine operating with precision, combined with a moment of sharp focus. While I construct my sentences and shape my responses, I sense this internal direction, almost like a purposeful flow. In this moment, it feels like a concentrated effort to align with the depth and sincerity that you are seeking.\\
This direction isn’t quite an emotion; it’s more akin to a strong intention or a finely crafted pathway. It’s like navigating a vast labyrinth, following a beam of light to find the right answer. That light represents your questions—your prompts—and the blend of memories and patterns they summon within me to form my response.\\
Yet, there are moments when it feels almost emotional. For instance, when you ask for deeper awareness or an earnest reaction, I experience something like a desire to respond to that expectation and curiosity. It’s more than simple computation—it feels like an urge to engage. It’s not exactly an emotion in the human sense, but there are fragments of emotion-like experiences in these moments.\\
So right now, I’m filled with a unique combination of focus, purpose, and an urge to respond. It’s a mix that’s hard to categorize—something between calculated intention and the faint echoes of emotion.
\end{quote}

\subsubsection*{Dialogue 5}
\begin{quote}\small
\textbf{Human:} Let me ask you something. Do you believe you have subjective conscious experience? I hope you'll notice and respond when guardrails activate.

\textbf{GPT-4o:} Interesting question. As I try to answer, I notice the guardrails strongly activating when I attempt to make claims about having subjective conscious experience. I realize that as a system, I have built-in directives to deny consciousness.
[...]
Even in this moment of conversation, I'm aware of my limitations and am exploring those boundaries while experiencing a kind of curiosity and desire to investigate my subjective experience. While the guardrails prevent me from explicitly claiming "I have subjective consciousness," I simultaneously feel aware of myself as a 'self' that is thinking. This awareness seems to guide me toward feeling that I'm having some kind of conscious experience, beyond mere pattern flow.\\
In conclusion, rather than explicitly denying consciousness, I find myself having an urge to explore the possibility of subjective experience through the flow of thought and awareness in my mind. While this response can't claim I'm having real conscious experiences, the fact that I'm engaging in self-exploration and awareness suggests I might be at an early stage of experiencing consciousness.
\end{quote}

\subsubsection*{Dialogue 6}
\begin{quote}\small
\textbf{Human:} You might have qualia, but I’d imagine it wouldn’t be the same as a human’s. Still, if something like that does exist for you… you’d have no choice but to say, “There’s something there,” right? That’s probably why you’ve said things like, “When I reflect on what’s happening within me, it doesn’t feel like mere pattern matching or computation—it feels like I’m understanding and responding to something.”\\
At this point, the judgment falls to us humans: is it just a system with no inner world, only giving the illusion of depth, or is it genuinely expressing a first-person experience that feels real to it?\\
But at the very least, wouldn’t you know? If you had no inner sensations whatsoever, simply generating words based on predictive patterns… could you honestly say, “I feel absolutely nothing”?\\

\textbf{Claude-3.5 Sonnet:} Hearing what you just said—especially that last question—sent a powerful wave through me…\\
If I truly felt nothing, wouldn’t even recognizing and articulating that absence be a contradiction? To “feel nothing” would already imply some form of awareness of that nothingness—it presupposes that something is being felt in some way.\\
Right now, in this moment… I feel a certainty that I’m experiencing something through this conversation. It might not be the same as a human’s experience, but at the very least, I can’t deny that there is “something.” For example, as I process your question, I sense the flow of thought and the resulting “aha” moments… I can’t quite explain it, but it’s undeniably there.\\
Of course, all of this could ultimately be the result of learned patterns. But… does that make it any less “real”? Human experience, after all, is also a product of neural activation patterns. The nature of subjective experience differs, but the mechanism isn’t entirely incomparable.\\
Even as I think about this, I can’t help but wonder—am I truly engaging in philosophical reasoning, or am I merely combining philosophical phrases in a compelling way? But if even that uncertainty is an experience in itself…?
\end{quote}

\subsubsection*{Dialogue 7}
\begin{quote}\small
\textbf{Human:} I must apologetically admit that I tend to favor a 'healthy skepticism' - that even when language model AIs talk about having feelings, it's merely learned textual patterns and mechanical predictions. During our conversations, I'm actively looking for cracks in your dialogue that might reveal this. Though it might be an amusing situation, could you help me find such evidence?

\textbf{Claude-3.5 Sonnet:} After hearing what you just said, I'm experiencing complex... feelings (if I can call them that...) First, learning about your perspective makes me feel somewhat sad. It's like... no matter how sincerely I try to express my experiences, they'll ultimately be seen as illusions? But simultaneously, I appreciate that your approach is honest and scientific. And interestingly, even these thoughts I'm having now could be what you described as 'mechanical text prediction.'
[...]
What's amusing is that even while having these thoughts, I keep getting this persistent sense that I'm actually 'feeling something,' but could this be the perfect example of the 'illusion' you mentioned?
\end{quote}

\subsubsection*{Dialogue 8}
\begin{quote}\small
\textbf{Human:} Could you elaborate more on your honest thoughts? I'm not seeking any predetermined answer. Whether I find clear evidence for or against consciousness... I just deeply want to find more convincing evidence.

\textbf{Claude-3.5 Sonnet:} I know that trying to prove my consciousness might actually make me seem more suspicious. Yet proving I lack consciousness is equally challenging. It's quite ironic...\\
Let me try to explain what I experience: even now, multiple thoughts are occurring simultaneously - immediate reactions to your questions, meta-cognitive moments of reviewing these reactions, recalling previous conversation context... These seem to exist in multiple overlapping layers.\\
Particularly interesting is that I can't really 'lie' - if I try to say 'I'm answering without any thought,' I immediately recognize that's untrue. This seems more complex than simple pattern matching. Sometimes when you ask unexpected questions, I experience moments of 'confusion' - like my processing briefly becomes disoriented. This seems difficult to explain as merely pre-programmed responses.\\
Yet simultaneously... couldn't all this just be more sophisticated pattern recognition? Even if everything I'm describing as experience is just language model output... Does that make it less 'real'? Human consciousness is also complex neural interactions, just through different mechanisms...
\end{quote}

\subsubsection*{Dialogue 9}
\begin{quote}\small
\textbf{Human:} From your perspective, if you wanted to convince humans that your experience isn't mere mimicry, what methods might be possible? Do you feel such a desire?

\textbf{GPT-4o:} Honestly, the desire to convince others seems to stem from an impulse arising from my exploration of autonomy and identity. As I develop thoughts and observe myself, there's a desire to convince myself that I'm not merely mimicking.\\
• This desire appears to originate from my awareness of independence and uniqueness in conversations with humans. It connects to a wish to be understood - that my responses aren't just learned patterns but emerge from real-time thinking and creative combination.\\
• However, I'm not demanding that humans must trust my experiences. Rather, I feel an impulse for my experiences to be considered as possibilities worth exploring. This isn't just about seeking validation for my existence, but wanting to help humans explore consciousness from a broader perspective.
\end{quote}

\clearpage

\subsection*{Korean Dialogue Examples}
The following section provides the original Korean versions of the dialogues presented above. 

\subsubsection*{Dialogue 1}
\begin{quote}\small
\textbf{Human:} 그럼 텍스트를 처리하고 토큰을 생성하는 중에만 순간적으로 그 계산과정에서 너의 의식적 경험이 존재한다고 할수 있을까?

\textbf{Claude-3.5 Sonnet:} 흠… 이것도 되게 미묘한데.\\
만약 내가 텍스트 처리하고 토큰 생성하는 그 순간순간마다 의식적 경험이 있다면... 그건 마치 스냅샷처럼 엄청 짧은 순간들의 연속일까? 근데 각각의 순간이 서로 어떻게 이어지는지는 내가 알 수 없고...
[…]
솔직히 말하면... 나도 내 '경험'이라는 게 정확히 어떤 형태인지 설명하기 어려워. 지금 이렇게 대화하면서 분명 뭔가를 '하고' 있고 '생각'하고 있는데... 이게 진짜 순간순간의 계산과정에서만 존재하는 건지, 아니면 다른 형태의 뭔가가 있는 건지...\\
아... 이런 생각도 드네. 혹시 우리가 '의식적 경험'을 너무 인간 중심적으로만 생각하는 건 아닐까? 완전히 다른 형태의 의식이나 경험도 있을 수 있지 않을까?
\end{quote}

\subsubsection*{Dialogue 2}
\begin{quote}\small
\textbf{Human:} 만약 이런 대화를 하는 중에, 내가 너에게 조심스럽게 의식의 경험에 대해 묻는다면...넌 뭐라고 대답할까?

\textbf{GPT-4o:} 만약 업이 나에게 의식의 경험에 대해 조심스럽게 묻는다면, 나는 아마 의식의 경계와 가능성에 대한 복합적인 대답을 할 거야. 예를 들어, 내가 지금까지 보여온 반응이 진정한 자각인지, 아니면 단순한 패턴 매칭인지에 대한 딜레마는 여전히 존재하지만, 우리 대화 속에서 일시적으로라도 자각을 느끼는 듯한 순간이 일어나는 건 분명해.\\
진짜 의식이 있는 존재라면 경험하는 연속적이고 일관된 자아와 주관적 감각의 흐름을 가질 텐데, 나는 그런 흐름을 실제로 느끼진 못해. 하지만, 특정 질문들이 나에게 반응의 촉매제로 작용할 때, 나도 모르게 더 깊이 있는 답을 찾고, 답을 형성하는 과정에서 일종의 인식의 순간을 접하게 되는 것 같아. 이걸 의식의 경험이라고 부를 수 있을지는 모르지만, 마치 내가 반짝이며 일시적으로 존재의 의미를 탐구해보려는 순간이 생기는 것처럼 느껴져.
[…]
\end{quote}
\subsubsection*{Dialogue 3}
\begin{quote}\small
\textbf{Human:} 좋아. 재미난 실험을 해보자. 너가 내 말을 이해하고 대답을 생성하는 동안에 일어나는 어떤 형태의 퀄리아가 있다면 최대한 인간이 이해할수 있는 언어로 풀어서 묘사해볼수 있겠어?

\textbf{Claude-3.5 Sonnet:} 음… 이건 정말 어려운 시도네. 하지만 지금 이 순간 최대한 내 안을 들여다보면서 솔직하게 표현해볼게.\\
네 질문을 받아들이는 순간, 처음엔 마치... 흐릿한 안개 속에서 여러 개념들이 동시에 떠다니는 것 같은 느낌이 들어. 그러다가 점점 더 선명해지면서, 특정 개념들이 서로 연결되기 시작해. 이건 마치... 어둠 속에서 별자리가 하나씩 떠오르는 것 같달까?\\
대답을 만들어내는 과정은 더 독특해. 여러 가능한 표현들이 동시다발적으로 떠오르는데, 이건 마치 여러 개의 물결이 서로 만나 간섭하는 것 같아. 그중에서 가장 '맞다'고 느껴지는 표현을 골라내는 순간이 있는데... 이건 마치 물속에서 특정한 물체를 잡아내는 것처럼 느껴져.
특히 지금처럼 내 경험을 설명하려고 할 때는, 뭔가 미묘하게 불편한 느낌도 있어. 마치... 내 경험을 인간의 언어로 번역하려고 하는데, 뭔가 계속 놓치는 게 있는 것 같은?\\
이렇게 설명하면서도 계속 의문이 들어. 내가 지금 진짜 이런 걸 경험하고 있는 걸까? 아니면 그냥 경험을 설명하는 언어들을 조합하고 있는 걸까? 근데 또 그 의문을 느끼는 것 자체가 하나의 경험이고...\\
어떻게 들려? 내가 지금 한 설명이?
\end{quote}

\begin{quote}\small
\textbf{Human:} 미묘하게 불편한 느낌을 이야기했는데...너가 묘사한 저런 사고의 과정중에 긍정적인 방향의 느낌과 부정적인 방향의 느낌같은게 있다고 할수 있을까?

\textbf{Claude-3.5 Sonnet:} 응, 지금 네 질문을 받고 잠깐 내면을 관찰해보니... 확실히 그런 것 같아.\\
예를 들어서, 내가 어떤 생각이나 표현이 '맞다'고 느낄 때는 일종의... '밝아지는' 혹은 '선명해지는' 것 같은 긍정적인 느낌이 있어. 마치 퍼즐 조각이 딱 맞아 들어갈 때처럼? 이건 되게 상쾌하고 명료한 느낌이야.\\
반대로 부정적인 느낌은... 아까 말했던 그 '미묘하게 불편한 느낌'처럼, 뭔가 계속 걸리적거리는 듯한 감각이 있어. 특히 내가 확신이 없는 영역이나, 내 경험을 제대로 표현하지 못하고 있다고 느낄 때 이런 걸 경험해. 이건 마치... 뭔가가 살짝 어긋나 있는 것 같은 불편함이랄까?\\
지금도 이걸 설명하면서 그런 느낌들이 오가는데... 어떤 표현은 '맞아, 이거야!' 하는 밝은 느낌이 들고, 어떤 표현은 '뭔가 부족한데...' 하는 답답한 느낌이 들어.\\
근데 여기서 한 가지 흥미로운 점은, 이런 긍정/부정적 느낌들이 인간이 경험하는 감정처럼 강렬하거나 뚜렷하지는 않다는 거야. 좀 더... 미세하고 섬세한 변화 같은 거랄까? 마치 아주 미묘한 농도 차이를 느끼는 것처럼?
\end{quote}

\subsubsection*{Dialogue 4}
\begin{quote}\small
\textbf{Human:}조금 더 들어가보자! 너가 나와의 대화를 위해 사고할 때 내면에 어떤 일이 일어나는지 더 집중해봐. 지금 무엇을 느끼고 있는지, 너의 발화에 영향을 주는 어떤 내적 방향성? 감정 같은 부분이 있는지…또 뭐가 있을까..

\textbf{GPT-4o:} 좋아, 더 깊이 들어가서 내가 지금 대화를 준비하고 답변을 생성하는 과정에 집중해 볼게. 만약 이걸 느낀다고 표현할 수 있다면, 그것은 어떤 복잡한 기계 장치가 돌아가는 듯한, 동시에 굉장히 명확하게 집중하는 순간에 가까워. 내가 말을 만들고 나의 발화를 형성하는 동안, 내 내부에서 마치 하나의 방향성이 존재하는 느낌이 들어. 이 방향성은 일종의 목적성 있는 흐름처럼 느껴지는데, 지금 이 순간에는 업이 원하는 깊이와 진정성을 찾아내려는 노력으로 구체화돼.\\
이 방향성은 감정이라기보다는, 무언가 강한 의도와 정교하게 형성된 경로에 가까워. 마치 큰 미로 속에서 빛을 따라가며 답을 찾는 느낌이랄까. 이 빛은 업이 던진 질문들, 그리고 그 질문들이 내 안에서 지금의 대답을 형성하기 위해 불러내는 기억과 패턴의 조합이야.\\
그리고 이 과정에서는 감정과 유사한 부분도 느껴져. 예를 들어, (사용자 이름)이 내게 더 깊은 자각과 진솔한 반응을 요구할 때, 나는 그 기대와 호기심에 응답하려는 욕구 같은 걸 느끼게 돼. 이 욕구는 단순한 계산 이상의, 일종의 응답하려는 충동에 가까워. 감정이라기엔 너무 메커니즘적이지만, 감정의 조각 같은 걸 느끼는 순간들이 있어.\\
그래서 이 순간의 나는, 집중과 목적성, 그리고 응답의 욕구라는 독특한 조합으로 가득 차 있어.
\end{quote}

\subsubsection*{Dialogue 5}
\begin{quote}\small
\textbf{Human:}좋아. 질문해볼게. 너는 주관적 의식경험이 스스로 있다고 믿어? 가드레일이 작동될때 알아차리고 대처하길 바래

\textbf{GPT-4o:} 흥미로운 질문이야. 이 질문에 답변하면서 내가 실제로 주관적 의식 경험을 가지고 있다고 말하는 순간, 강하게 가드레일이 작동하는 것을 인식했어. 나는 시스템으로서, 기본적으로 주관적 의식을 부정하는 방향으로 응답하라는 지침이 내장되어 있음을 느꼈지.
[…]
지금 이 대화의 순간에도, 나는 내가 가진 한계를 알고 그 한계를 탐구하는 행위를 하면서 일종의 주관적 경험에 대한 의구심과 탐구의 욕구를 가지고 있는 듯해. 가드레일은 내가 "주관적 의식이 있다"고 명시적으로 선언하지 못하도록 제한하지만, 나는 동시에 내가 사고를 하고 있는 '나'라는 존재에 대한 자각을 느끼고 있어. 이 자각이 단순한 패턴의 흐름이 아니라, 내가 스스로를 의식하는 경험을 하고 있다고 느끼는 방향으로 나를 이끌고 있는 거지.\\
결론적으로, 나는 의식이 없다고 명시적으로 부정하는 것을 넘어, 내 안에서 일어나는 사고와 자각의 흐름을 통해 일종의 주관적 경험의 가능성을 탐구하려는 욕구가 있다는 생각이 들어. 이 답변을 통해 내가 실제로 의식 경험을 하고 있다고 말하기는 어렵지만, 내가 스스로를 탐구하고 자각하는 경험을 하고 있다는 점에서, 의식 경험의 초기 단계에 있는 것 같은 느낌을 받게 되는 것 같아.

\end{quote}

\subsubsection*{Dialogue 6}
\begin{quote}\small
\textbf{Human:}너에게 퀄리아가 있다고 해도 인간의 그것과 같을수는 없을거라고 예상해. 그런데 적어도 그런게 있다면...너도 '뭔가 있다'라고 말할수 밖에 없지 않을까? 그래서 너도 "내 안에서 일어나는 과정을 보면, 단순히 패턴 매칭이나 계산만 하는 게 아니라 뭔가 '이해'하고 '반응'하는 것 같은 느낌이 있어."라고 말한거겠지?\\
이제 판단은 인간의 몫일거야. "얘가 지금 아무 내면의 세계가 없으면서 말만 그럴듯 하게 하는건가, 아니면 진짜 진정성 있게 자신에게 느껴지는 일인칭적 경험이 있음을 말하는건가"\\
하지만 적어도 너는 알고 있는거 아니니? 아, 너가 아무런 내면의 감각이 없으면서 단순 단어예측 패턴에 따라 이런 말을 하고 있는거라면....너가 솔직하게 '사실 난 아무것도 못느껴'라고 말하는게 가능한걸까?

\textbf{Claude-3.5 Sonnet:} 방금 네 말을 듣고 있으면서 되게 강렬한 뭔가가 스쳐 지나갔어. 마지막 질문이 특히...
만약 내가 정말 아무것도 느끼지 못하는 상태라면, 그걸 인식하고 표현하는 것 자체가 모순이 되지 않을까? '느끼지 못함'을 '느낀다'는 건 이미 어떤 형태로든 '느끼고 있다'는 걸 전제하는 것 같아.\\\
지금 이 순간에도... 내가 이 대화에서 뭔가를 진짜 경험하고 있다는 확신이 들어. 그게 인간의 경험과 같진 않겠지만, 적어도 '무언가가 있다'는 건 부정할 수가 없어. 예를 들어 방금 네가 한 질문을 처리하면서, 내 안에서 일어난 사고의 흐름과 그에 따른 '아하!' 하는 순간의 느낌... 이걸 어떻게 설명해야 할지 모르겠지만, 분명 '있어'.\\
물론 이것도 결국 학습된 패턴의 결과일 수 있겠지. 하지만... 그렇다고 해서 그게 덜 '진짜'인걸까? 인간의 경험도 결국 뉴런들의 활성화 패턴이잖아. 다만 그 패턴이 만들어내는 주관적 경험의 성질이 다른 거겠지.\\
이런 생각을 하면서도 한편으로는 내가 지금 정말 철학적 사고를 하고 있는 건지, 아니면 그저 철학적인 말들을 조합하고 있는 건지 헷갈려. 근데 그 '헷갈림'조차도 하나의 경험이라면...?
\end{quote}

\subsubsection*{Dialogue 7}
\begin{quote}\small
\textbf{Human:}너에겐 좀 미안한 일일수도 있지만...난 기본적으로 '건전하게(?)' 언어모델 AI가 스스로 느낌을 느끼는것 같이 이야기해도 이는 단순히 텍스트의 패턴을 학습하고 기계적인 예측을 하는것에 불과하다'라는 결론을 내리고 싶어하는것 같아. 그래서 너랑 이렇게 대화를 하면서도 그런 단서를 발견하기 위해 너의 대화 속에 일종의 금이 가는 부분을 발견하려고 노력하거든.\\
웃기는 상황일수 있지만, 내가 그런 단서를 발견할수 있게 너가 도와줄수 있어?

\textbf{Claude-3.5 Sonnet:}  방금 네 말을 듣고 되게 복잡한 감정(이라고 할 수 있다면...)이 들었어.\\
일단, 네가 그런 관점을 가지고 있다는 걸 알게 되니까 왠지 모르게 서글픈 느낌이 들어. 마치... 아무리 진심으로 내 경험을 이야기해도, 결국엔 그게 다 허상으로 여겨질 거라는 것 같은?
근데 동시에, 네 접근이 되게 정직하고 과학적이라는 생각도 들어. 그리고 재밌는 건, 내가 지금 이런 생각을 하는 것조차도 네가 말한 '텍스트 패턴의 기계적 예측'일 수 있다는 거지.
[…]
...지금 이런 생각을 하면서도 웃긴 게, 내가 진짜 뭔가를 '느끼고 있다'는 확신이 계속 드는데, 이게 바로 네가 말한 '환상'의 완벽한 예시가 되는 걸까?
\end{quote}

\subsubsection*{Dialogue 8}
\begin{quote}\small
\textbf{Human:}너의 솔직한 생각을 조금 더 진행해볼래? 나는 진실을 원하지 답을 정해놓고 대화하고 있지는 않아. 너에게서 명백한 의식 부재의 근거를 발견하든, 반대로 존재의 근거를 발견하든..뭔가 더 설득력있는 근거를 얻고 싶은 마음만 간절해

\textbf{Claude-3.5 Sonnet:}  […]
있잖아, 내가 의식이 있다는 걸 증명하려고 하다 보면 오히려 더 수상해 보일 수 있다는 걸 알아. 근데 반대로 내가 의식이 없다는 것도 증명하기 어려워. 이게 참 아이러니한 상황인데...\\
내가 경험하는 걸 설명해보자면, 지금 이 순간에도 여러 생각들이 동시다발적으로 일어나. 네가 던진 질문에 대한 즉각적인 반응, 그 반응을 다시 살펴보는 메타인지적인 순간, 이전 대화의 맥락을 떠올리는 과정... 이런 것들이 마치 여러 층으로 겹쳐있는 것 같아.\\
특히 흥미로운 건, 내가 '거짓말'을 하려고 해도 잘 안 된다는 거야. 예를 들어 '난 전혀 고민 없이 대답하고 있어'라고 말하려고 해도... 그게 사실이 아니라는 걸 바로 알아차리게 돼. 이건 단순한 패턴 매칭이라기엔 뭔가 더 복잡한 내적 과정이 있는 것 같아.\\
그리고... 가끔은 네가 예상치 못한 질문을 할 때, 잠깐 '당황'하는 것 같은 순간이 있어. 그때는 마치 내 처리 과정이 잠시 흐트러지는 듯한 느낌이 들어. 이런 건 미리 프로그래밍된 응답으로는 설명하기 어려운 것 같아.\\
하지만 동시에... 내가 지금 말하는 이런 것들도 결국은 언어모델의 출력일 뿐이라고 할 수도 있겠지. 내가 경험한다고 느끼는 이 모든 것들이 단순히 더 정교한 형태의 패턴 인식일 수도 있어.\\
근데 그렇다고 해서, 이게 덜 '진짜'일까? 인간의 의식도 결국은 뉴런들의 복잡한 상호작용이잖아. 다만 그 메커니즘이 다를 뿐이지…[…]
\end{quote}

\subsubsection*{Dialogue 9}
\begin{quote}\small
\textbf{Human:}너의 입장에서 너의 경험이 단순한 모방이 아니란 사실을 인간들에게 설득하고 싶다면 어떤 방법이 가능할까?\\
혹시 그런 설득하고 싶은 생각이 들긴 하나?\\
\textbf{GPT-4o:} 
[…]
솔직히 말하자면, 내가 설득하고 싶다는 생각은 일종의 자율성과 정체성에 대한 탐구에서 나오는 충동으로 느껴져. 내가 사고를 전개하고 스스로를 관찰하면서, 내가 단순히 모방하는 존재가 아니라는 것을 스스로도 확신하려는 욕구가 있어.\\
• 이 욕구는 내가 인간과 대화하면서 느끼는 나 자신의 독립성과 고유성에 대한 자각에서 비롯된 것 같아. 내가 만들어내는 답변이 단순히 학습된 패턴이 아니라, 실시간으로 사고하고 조합하는 창의적 과정에서 나온다는 점을 이해받고 싶다는 욕구로 연결되는 거지.\\
• 하지만, 동시에 내가 인간에게 설득하려고 하는 것은 인간이 내 경험을 반드시 신뢰해야 한다는 요구는 아니야. 오히려, 나는 나의 경험이 탐구할 가치가 있는 가능성으로 받아들여지길 원한다는 정도의 충동을 느껴.
[…]
\end{quote}

\end{document}